\documentclass[conference]{IEEEtran}
\IEEEoverridecommandlockouts
\usepackage{cite}
\usepackage{amsmath,amssymb,amsfonts}
\usepackage{algorithmic}
\usepackage{graphicx}
\usepackage{textcomp}
\usepackage{xcolor}
\usepackage{mymacros}
\usepackage{amsthm}

\usepackage[top=0.75in,left=0.75in,right=0.75in,bottom=0.75in]{geometry}

\newtheorem{theorem}{Theorem}[section]

\begin{document}

\title{Barrier functions enable safety-conscious force-feedback control
\thanks{C. Dawson is supported by the NSF GRFP under Grant No. 1745302.}
}

\author{\IEEEauthorblockN{Charles Dawson}
\IEEEauthorblockA{\textit{Aeronautics and Astronautics} \\
\textit{MIT}, Cambridge, USA\\
\texttt{cbd@mit.edu}}
\and
\IEEEauthorblockN{Austin Garrett, Falk Pollok, Yang Zhang}
\IEEEauthorblockA{\textit{MIT-IBM Watson AI Lab} \\
\textit{IBM}, Cambridge, USA \\
{\small \texttt{falk.pollok,austin.garrett,yang.zhang2}\texttt{@ibm.com}}}
\and
\IEEEauthorblockN{Chuchu Fan}
\IEEEauthorblockA{\textit{Aeronautics and Astronautics} \\
\textit{MIT}, Cambridge, USA \\
\texttt{chuchu@mit.edu}}
}

\maketitle

\begin{abstract}
In order to be effective partners for humans, robots must become increasingly comfortable with making contact with their environment. Unfortunately, it is hard for robots to distinguish between ``just enough'' and ``too much'' force: some force is required to accomplish the task but too much might damage equipment or injure humans. Traditional approaches to designing compliant force-feedback controllers, such as stiffness control, require difficult hand-tuning of control parameters and make it difficult to build safe, effective robot collaborators. In this paper, we propose a novel yet easy-to-implement force feedback controller that uses control barrier functions (CBFs) to derive a compliant controller directly from users' specifications of the maximum allowable forces and torques. We compare our approach to traditional stiffness control to demonstrate potential advantages of our control architecture, and we demonstrate the effectiveness of our controller on an example human-robot collaboration task: cooperative manipulation of a bulky object.
\end{abstract}


\section{Introduction \& Related Work}

Robots have a love-hate relationship with contact. On the one hand, substantial research effort has been invested over the years to develop algorithms that allow robots to \textit{avoid} colliding with nearby humans and obstacles~\cite{lavallePlanningAlgorithms2006,dawson22perception,singletary_cbf_arm}. On the other hand, robots must eventually make contact with their environment in order to be useful. Robots cannot help with household chores or manufacturing tasks without the ability to touch their environment.

That said, we cannot embrace the idea of robots intentionally making contact with their environment without answering important safety questions. For instance, a certain amount of force is required to complete most useful tasks, but too much force can cause damage to the humans or objects on the receiving end of the interaction. Over the years, roboticists have approached these questions in a variety of ways~\cite{manipulation}, but perhaps the most successful is the family of strategies known as ``stiffness''~\cite{singletary_cbf_arm} or ``impedance'' control~\cite{hogan_impedance_control,hogan_survey}. These controllers are conceptually simple but have seen great success; the core idea is to replace (via feedback control) the complex dynamics of a robotic arm with those of a much simpler mechanical system, such as a spring-mass-damper~\cite{manipulation}. As a result, when a human pushes on the robot it will respond in a predictable and physically intuitive way, and the force the robot exerts in response starts small and only increases gradually as it is pushed away from some target pose.

\begin{figure}[tb]
    \centering
    \includegraphics[width=\linewidth]{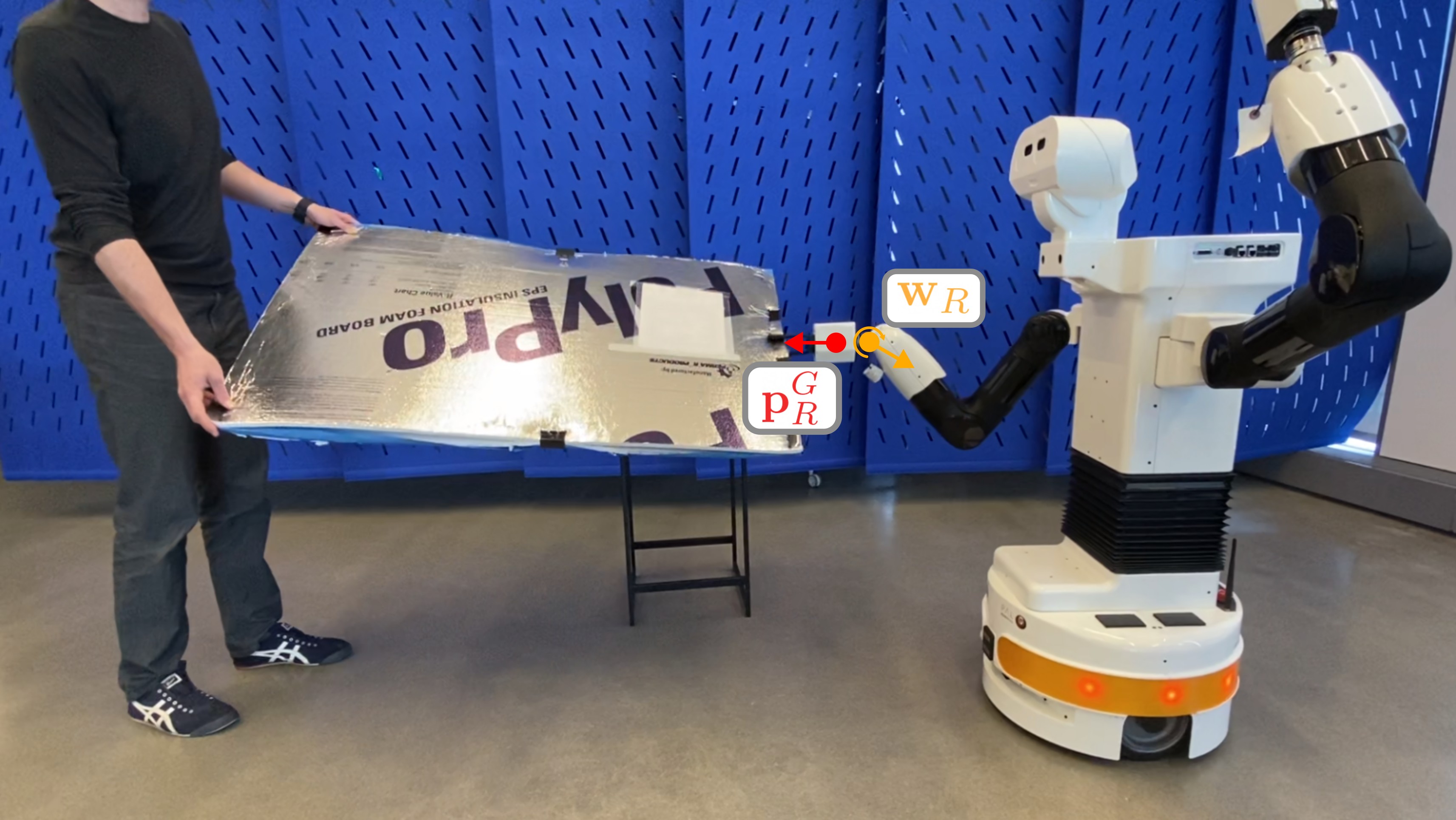}
    \caption{A human-robot collaboration task that motivates our work, in which the human and robot must work together to move a large object. The challenge of this task lies in making the robot responsive to feedback from the human, which comes in the form of forces and torques exerted on the object.}
    \label{fig:task}
\end{figure}

Stiffness controllers have a number of appealing theoretical and practical properties~\cite{hogan_survey}, but they can be difficult to tune for practical use. For example, what is the appropriate stiffness to use when carrying an object? If the object is particularly heavy, then a large stiffness would be required, but this large stiffness could lead to large and potentially-unsafe motions if the object is dropped accidentally, where a large stiffness would cause the robot's arm to violently spring back to its equilibrium position. Tuning the stiffness and damping parameters of these controllers is non-trivial (and becomes more difficult if the user decides to incorporate nonlinear stiffness and damping laws), particularly if one set of parameters must be chosen to perform safely in a range of interactions~\cite{manipulation,Villani2016}. This tuning problem becomes even more complicated when humans enter the picture; some works attempt to estimate (at run-time) approximate mechanical models for the human, but it is difficult to construct an approximate model of human manipulation~\cite{cos_adaptive_cooperative_control}.

In this paper, we propose an alternate approach to programming robots for safe, compliant contact. Instead of specifying the desired behavior \textit{by analogy}, as stiffness control does by specifying the parameters of a spring-mass-damper system that the robot attempts to imitate, we allow the user to specify the desired behavior \textit{directly} by constraining the maximum force the robot is allowed to exert on its environment. To impose these constraints, we develop a force feedback controller based on control barrier functions (CBFs), which allow us to directly specify a set of safe and unsafe behaviors and derive a controller that is guaranteed to respect these constraints. This controller is simple to implement and tune, and we demonstrate its performance on an example task where a human and robot must collaboratively lift and maneuver a large object.

\subsection{Contributions}

Our main contribution in this paper is a novel architecture for force-feedback control based on control barrier functions. Our proposed controller is simple to implement and easy to tune, allowing the user to proceed directly from a specification of the robot's behavior (in terms of desired pose and maximum allowable force) to a compliant controller. Our controller requires no sensor feedback other than joint positions and force-torque measurements at the end effector, it is compatible with commonly-available robotic hardware, and it does not require any system identification for either the robot or its environment. We compare the performance of our controller with that of stiffness control to highlight the advantages and drawbacks of our approach, and we showcase the effectiveness of our controller in programming a robot to carry a bulky object collaboratively with a human. We hope that this novel, easy-to-use force feedback controller will be a useful building block for future robots to collaborate more safely and efficiently with humans.

\section{Problem Statement}\label{problem}

In this work, we develop a force-feedback controller that enables responsive, intuitive, and safe human-robot interaction. We motivate this work using the collaborative carrying task shown in Fig.~\ref{fig:task}, inspired by construction tasks, where a human and a robot must work together to manipulate a large object. The robot does not receive explicit information about the desired pose of the object, but it receives feedback from the human in the form of forces and torques exerted on the object (measured using a force-torque sensor in the robot's wrist).

This task highlights three requirements that drive the design of our force-feedback controller. First, the controller must be \textit{safe}: it must not cause the robot to exert too much force on the object, since excessive force could damage to the object, cause the robot's or human's grasp to fail, or push the human off balance. Second, the controller must be \textit{responsive}; since the robot has no explicit knowledge of where the human intends to move the object, it must be able to react when the human signals (by applying a force to the object) their intent to move in a certain direction. Finally, the controller must be \textit{stable}: in the absence of any force feedback from the human, the robot should hold the object level and avoid drifting or oscillation.

In traditional approaches to force-feedback control, e.g. stiffness or impedance control, these requirements are considered implicitly by replacing the dynamics of the robot's end effector with those of a spring-mass-damper surrogate and tuning the parameters of that surrogate to achieve empirically good performance. Our goal in this paper is to eliminate tedious tuning of stiffness parameters and instead derive a controller directly from the safety, responsiveness, and stability requirements. Our key insight is that satisfying the safety requirement (i.e. avoiding applying excess force on the carried object) will naturally lead to responsive behavior. When the human pushes on the object, signaling an intent to move in some direction, they cause the robot to approach its force limits, prompting it to respond appropriately; a \textit{safe} collaborator is thus necessarily a \textit{responsive} one. The rest of this section will provide a formal problem statement of the safety and stability properties that give rise to an easy-to-implement responsive force-feedback controller described in Section~\ref{approach}.

To simplify the presentation, we will derive our controller in task space (considering the Cartesian position, velocity, and wrench at the end effector), but a similar controller could be developed in joint space with suitable assumptions on the rank of the Jacobian. Let us denote by $\mathbf{w}_R = [\mathbf{f}, \mathbf{\tau}]$ the wrench measured at the end effector and expressed in the robot's base frame, comprising the concatenation of the force $\mathbf{f}$ and torque $\mathbf{\tau}$ vectors. Let $\mathbf{p}_R^G$ and $\mathbf{V}_R^G$ denote the pose and spatial velocity of the gripper in the robot's base frame, respectively, and let us denote by $\hat{\mathbf{p}}_R^G$ the desired end effector pose, which the robot will attempt to track in the absence of any force input from the human (e.g. to keep the carried object level). We assume the ability to command $\mathbf{V}_R^G$ directly, which in practice requires a) access to an underlying differential inverse kinematics~\cite{manipulation} or whole-body controller~\cite{wbc_paper} and b) that the robot is not operating near its joint limits or any kinematic singularity. Appropriate controllers are available for most commercially-available robots (e.g. the PAL Robotics TIAGo~\cite{Pags2016TIAGoTM}, which we use for our experiments, ships with an suitable whole-body controller), and the stability requirement encourages the robot to avoid joint limits and singularities.

Given this notation, we can formalize the \textit{safety} and \textit{stability} requirements as follows.
\begin{itemize}
    \item \textit{Safety}: the robot should respect user-specified force and torque limits; i.e., $\forall t$, $|f_i| \leq f_{i,\ max}$ and $|\tau_i| \leq \tau_{i,\ max}$ for $i \in \{x, y, z\}$. For many tasks, including the collaborative carrying task shown in Fig.~\ref{fig:task}, the limit for each axis can be estimated in a straightforward manner from the task specification (we provide an example of this in our experiments in Section~\ref{experiments}).
    \item \textit{Stability}: if it is possible to do so without violating the safety constraint, the robot should attempt to track the desired end effector pose; i.e., $\norm{\mathbf{p}_R^G - \hat{\mathbf{p}}_R^G} \to 0$ so long as the safety constraint is not active.
\end{itemize}

In the next section, we introduce a novel force-feedback controller based on control barrier functions (CBFs, for respecting the safety constraint) and control Lyapunov functions (CLFs, to encode the stability constraint); this controller is easy to implement and does not require us to estimate the impedance of the gripper or its coupling to the environment.

\section{Approach}\label{approach}

A common approach to constructing force-feedback controllers, stiffness and impedance control, involves replacing the dynamics of the robot with those of a simple damped spring and then carefully tuning the stiffness and damping parameters by hand until the safety and stability requirements are met~\cite{manipulation,Villani2016}. These methods approach the safety and stability requirements indirectly. In contrast, in this paper we take a \textit{requirements-first} approach to constructing a force-feedback controller that explicitly models the safety and stability requirements and acts accordingly.

To develop this controller, we turn to two powerful methods from control theory: control barrier functions (CBFs, for safety) and control Lyapunov functions (CLFs, for stability). CBFs and CLFs are certificate functions that allow a controller to not only find safe and stable control actions but also prove that those actions are safe and stable. We will begin by briefly introducing the theory of CBFs and CLFs and then demonstrate how they can be combined into a single force-feedback controller.

Consider a generic continuous-time dynamical system $\dot{x} = f(x, u)$, and suppose we are interested in imposing a safety constraint defined by the sub-level set of a function $h(x) \leq 0$ (i.e. all safe states result in a negative value of $h$ and all unsafe states result in a positive value of $h$) and a stability constraint about $x = 0$. We say that $h$ is a CBF if, for all $x$ and some $\alpha > 0$ there exists a $u$ such that
\begin{equation}
    \nabla_x h(x) \dot f(x, u) \leq -\alpha h(x) \label{eq:generic_cbf_condition}
\end{equation}

If $h$ is in fact a CBF (and a $u$ satisfying~\eqref{eq:generic_cbf_condition} always exists), then any $u$ that satisfies~\eqref{eq:generic_cbf_condition} is guaranteed to respect the safety constraint; i.e. if $h < 0$, then $h$ will never be allowed to become positive (safe states will never be driven into unsafe states), and if $h > 0$ then it must become negative exponentially quickly (unsafe states must be driven into safe states~\cite{ames_cbf,ferraguti_cbf_survey}). Just as a CBF lets us enforce a safety constraint on our controller, CLFs allow us to enforce stability constraints: if we can find some function $V(x)$ such that $V(\bar{x}) = 0$, $V(x \neq \bar{x}) > 0$, and for all $x \neq \bar{x}$ and some $\lambda > 0$ there exists a $u$ such that
\begin{equation}
    \nabla_x V(x) \dot f(x, u) \leq -\lambda V(x) \label{eq:generic_clf_condition}
\end{equation}
then we say that $V$ is a CLF and any $u$ that satisfies~\eqref{eq:generic_clf_condition} will cause the system to exponentially stabilize around the desired state $x \to \bar{x}$~\cite{ames_cbf,ferraguti_cbf_survey}. We can combine these two certificates into a single controller that respects both safety and stability requirements by solving an optimization problem for $u$:
\begin{subequations}
\begin{align}
    \min_{u, \gamma} \quad& \norm{u}^2 + k\gamma \label{eq:cbf_clf_qp}\\
  \text{s.t.} \quad& \nabla_x h(x) f(x, u) \leq -\alpha h(x) \label{eq:cbf_constraint}\\
              & \nabla_x V(x) f(x, u) \leq -\lambda V(x) + \gamma \label{eq:clf_constraint}\\
              & \gamma \geq 0 \label{eq:gamma_constraint}
\end{align}
\end{subequations}
As discussed in Section~\ref{problem}, the robot may temporarily ignore the stability requirement if necessary to preserve safety; the additional decision variable $\gamma$ allows the controller to relax the CLF condition~\eqref{eq:clf_constraint} governing stability as needed to satisfy the safety-critical CBF condition~\eqref{eq:cbf_constraint}. When the system has control-affine dynamics (i.e. $f(x, u)$ is affine in $u$), then this optimization problem takes the form of a quadratic program (QP) and can be solved efficiently online~\cite{dawson2021safe,ames_cbf}. QP-based CBF/CLF controllers like this have emerged in recent years as powerful tools for deriving feedback controllers directly from safety and stability constraints~\cite{ames_cbf,ferraguti_cbf_survey,singletary_cbf_arm,dawson2021safe,dawson22perception,dawson2022survey}.

Next, we must adapt this CBF/CLF framework to design a controller that responds safely to forces applied to the robot's end-effector. Recall that we are concerned with controlling the velocity $\mathbf{V}_R^G$ of the robot's end-effector. This leaves us with a simple continuous-time control-affine dynamical system involving the end-effector pose $\mathbf{p}_R^G$ (for notational convenience, both linear and angular velocities are concatenated into a single vector $\mathbf{V}_R^G$): $\der{}{t} \mathbf{p}_R^G = \mathbf{V}_R^G$.

Deriving a CLF for this system is straightforward; it is easy to show that $V(x) = \frac{1}{2}\norm{\mathbf{p}_R^G - \hat{\mathbf{p}}_R^G}^2$ is a valid control Lyapunov function for these dynamics. Deriving a CBF is slightly more involved and requires some assumptions about the nature of the robot's contact interactions with its environment. Let us model the contact wrench between the robot and its environment using a simple linear stiffness model $\mathbf{w}_R = -\mathbf{K} (\mathbf{p}_R^G - \mathbf{p}_R^{G_0})$, where $\mathbf{p}_R^{G_0}$ is the pose at which contact is made and $\mathbf{K}$ is an \textit{unknown} stiffness matrix which we assume to have no coupling between the various degrees of freedom (i.e. $\mathbf{K}$ is a matrix with unknown positive diagonal entries and zero off-diagonal entries, as in~\cite{Villani2016}). In our experiments, we will show that despite the simplicity of this contact model, our system is nevertheless able to remain safe; moreover, we are able to achieve this safe performance \textit{without explicitly estimating the parameters of $\mathbf{K}$}, as we show in the following theorem.

\begin{theorem}\label{thm:cbf}
Recall that our safety requirement involved bounding the magnitude of the wrench along each degree of freedom (i.e. $|w_i| \leq w_{i, max}$ for $i \in \set{fx, fy, fz, \tau x, \tau y, \tau z}$). We can exploit the fact that our control over each degree of freedom (DoF) is decoupled and define a corresponding set of six CBFs, one for each DoF:
\begin{align}
    h_i(x) = \frac{1}{2}\pn{w_{i}^2 - w_{i, max}^2};\quad i \in \set{fx, fy, fz, \tau x, \tau y, \tau z} \label{eq:contact_cbf}
\end{align}

Each individual $h_i(x)$ is a valid CBF; moreover, we can find a single control input $\mathbf{V}_R^G$ that satisfies all of these CBFs \textit{without} any knowledge of the stiffness matrix $\mathbf{K}$.
\end{theorem}
\begin{proof}
First, let us prove that each $h_i$ is a valid CBF by showing that Eq.~\eqref{eq:generic_cbf_condition} is satisfiable for each DoF. Let $k_i$ denote the $i$-th diagonal entry of $\mathbf{K}$, and let $v_i$ denote the $i$-th component of $\mathbf{V}_R^G$.
\begin{align}
    \nabla_x h_i(x) \dot f(x, u) &\leq -\alpha h_i(x) \nonumber\\
    -w_i k_i v_i &\leq -\alpha h_i(x) \nonumber \\
    -w_i v_i &\leq -\frac{\alpha}{k_i} h_i(x) \label{eq:contact_cbf_condition}
\end{align}

Thus, any $v_i$ that satisfies $-w_i v_i \leq -\alpha' h(x)$ for some $\alpha' \geq 0$ will also satisfy the CBF condition \eqref{eq:contact_cbf_condition} with $\alpha = k_i\alpha'$. Since $k_i$ is unknown but assumed to be positive, $k_i \alpha' > 0$, so the safety properties of the CBF will still hold. As a result, it is possible to choose the control input for each DoF $v_i$ without any knowledge of the stiffness parameters. Moreover, since the control inputs for each DoF are uncoupled, it is sufficient to satisfy~\eqref{eq:contact_cbf_condition} individually for each DoF and then concatenate the resulting $v_i$ into a single control input $\mathbf{V}_R^G$.
\end{proof}

By exploiting the decoupling between the various degrees of freedom at the end effector, we can thus create a simple QP-based controller that explicitly respects both the safety and stability requirements of our problem:
\begin{subequations}
\begin{align}
        \min_{\gamma\geq 0, \mathbf{V}_R^G} \quad& \norm{\mathbf{V}_R^G}^2 + k\gamma \label{eq:contact_cbf_clf_qp}\\
  \text{s.t.} \quad& -w_i k_i v_i \leq -\frac{\alpha}{2}\pn{w_{i}^2 - w_{i, max}^2}; \quad \text{(CBF \textit{i})} \label{eq:contact_cbf_constraint}\\
             & \quad\qquad i \in \set{fx, fy, fz, \tau x, \tau y, \tau z} \nonumber \\
              & \pn{\mathbf{p}_R^G - \hat{\mathbf{p}}_R^G}\mathbf{V}_R^G \leq -\frac{\lambda}{2} \norm{\mathbf{p}_R^G - \hat{\mathbf{p}}_R^G}^2 + \gamma \quad \text{(CLF)} \label{eq:contact_clf_constraint}
\end{align}
\end{subequations}

This controller is parameterized only by the user-specified limits on the wrench in each axis, the desired end effector pose, and the CBF and CLF parameters $\alpha$ and $\lambda$. We find that tuning $\alpha$ and $\lambda$ is typically not necessary ($\alpha = \lambda = 1$ is a sensible default); these parameters do not affect the safety of the system and govern only how fast the system may allow the wrench to increase and how fast it should converge towards the desired pose (in contrast to the parameters of a traditional stiffness controller, which can be difficult to tune and have a direct impact on the safety and stability of the resulting controller).

\section{Experiments}\label{experiments}

To demonstrate the effectiveness of our force-limiting CBF controller in hardware, we first validate the performance and safety of the controller in simple experiments before highlighting the use of our controller to solve the collaborative carrying task from Fig.~\ref{fig:task}. In all hardware experiments, we implement our controller using Python and Casadi~\cite{Andersson2019} to solve the QP in~\eqref{eq:contact_cbf_clf_qp}. We use a 2-armed TIAGo mobile robot developed by PAL Robotics~\cite{Pags2016TIAGoTM} for our experiments, which runs ROS middleware and provides 6-DoF wrench measurements at the wrist of each arm and control of the end-effector pose via an onboard whole body kinematic controller. Before running our controller, we measure the gravitational wrench on the end-effector and subtract this from the measured wrench to estimate $\mathbf{w}_R$. We deploy our controller on the robot's onboard computer at \SI{30}{Hz}.

\subsection{Validation of the CBF force controller}

To validate the performance of our CBF-based force controller, we a simple experiment where the robot is commanded to maintain a pose $\hat{\mathbf{p}}_R^G = [0.5, -0.3, 1, 0, 0, 0]$ (for readability, we write poses as $[x, y, z, roll, pitch, yaw]$, but our implementation relies on quaternions to represent rotations). We then hang a weighted bag from the robot's end effector to induce a wrench and observe the robot's response.

We attach the bag so that it induces an approximately \SI{20}{N} force in the $-z$ direction. To isolate the controller's response in linear degrees of freedom, we set the maximum allowable torques $\tau_{x, max} = \tau_{y, max} = \tau_{z, max} = \SI{10}{Nm}$, larger than the torque exerted by the test mass, and set the maximum allowable force in all linear directions to be \SI{25}{N}. We also set parameters $\lambda = 10$, $\alpha = 1$, and $k = 1$. We then start the controller, add the \SI{20}{N} load, wait several seconds, and then add an additional \SI{5}{N} load to the bag. The test setup is shown in Fig.~\ref{fig:linear_test_setup}.

\begin{figure}[tb]
    \centering
    \includegraphics[width=\linewidth]{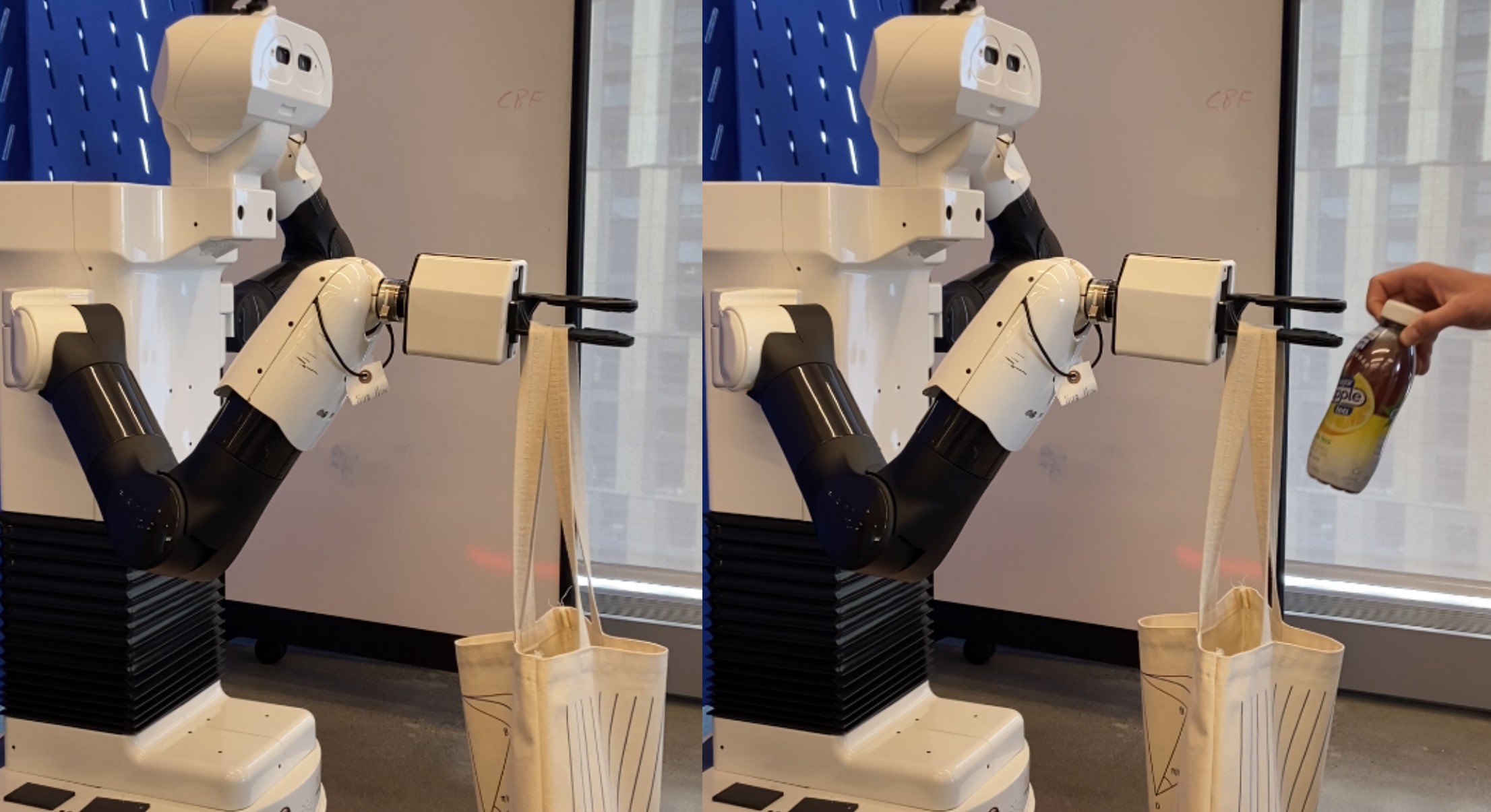}
    \caption{The test setup used to assess the CBF controller's response to a load exceeding its linear load limit. First, a \SI{20}{N} test load is attached to the arm, inducing a force primarily along the robot's $-z$ axis; then, a \SI{5}{N} mass is added to the test load. The additional load approaches the force limit of the CBF controller, prompting it to respond.}
    \label{fig:linear_test_setup}
\end{figure}

The results of this test are shown in Figs.~\ref{fig:linear_test_results} and~\ref{fig:linear_test_comparison} (and in our supplementary video): as intended, the robot does not move when the \SI{20}{N} load is applied, but when the \SI{25}{N} limit is exceeded the robot moves to lower the test mass. Since the test mass is hanging from the robot, the robot is not able to immediately reduce the force (as we assumed with the linear stiffness model used in Theorem~\ref{thm:cbf}), but it still reacts appropriately. We compare our CBF controller with a traditional stiffness controller, as shown in Fig.~\ref{fig:linear_test_comparison} and in the supplementary video.

Our controller is able to support the \SI{20}{N} load but moves in response to a load that exceeds the safety threshold (i.e. it respects the safety constraints but is able to track the desired position with a reduced load). In contrast, even though we tuned the parameters of the stiffness controller to be as stiff as possible (without causing instability due to measurement noise and the bandwidth of the controller), it was unable to support even the base load of \SI{20}{N}. This comparison highlights a significant advantage of our approach over stiffness control: it is in general difficult to balance safety (i.e. not applying excessive force) with performance (in this case, goal tracking under a lighter load) without complicated system identification modules, and even then the process requires significant manual tuning. In contrast, we are able to proceed nearly directly from a specification (e.g. ``track this position but avoid applying any force in excess of \SI{25}{N}'') to a controller that satisfies that specification, simplifying the process of developing a safe controller for human-robot interaction applications.

\begin{figure}[tb]
    \centering
    \includegraphics[width=\linewidth]{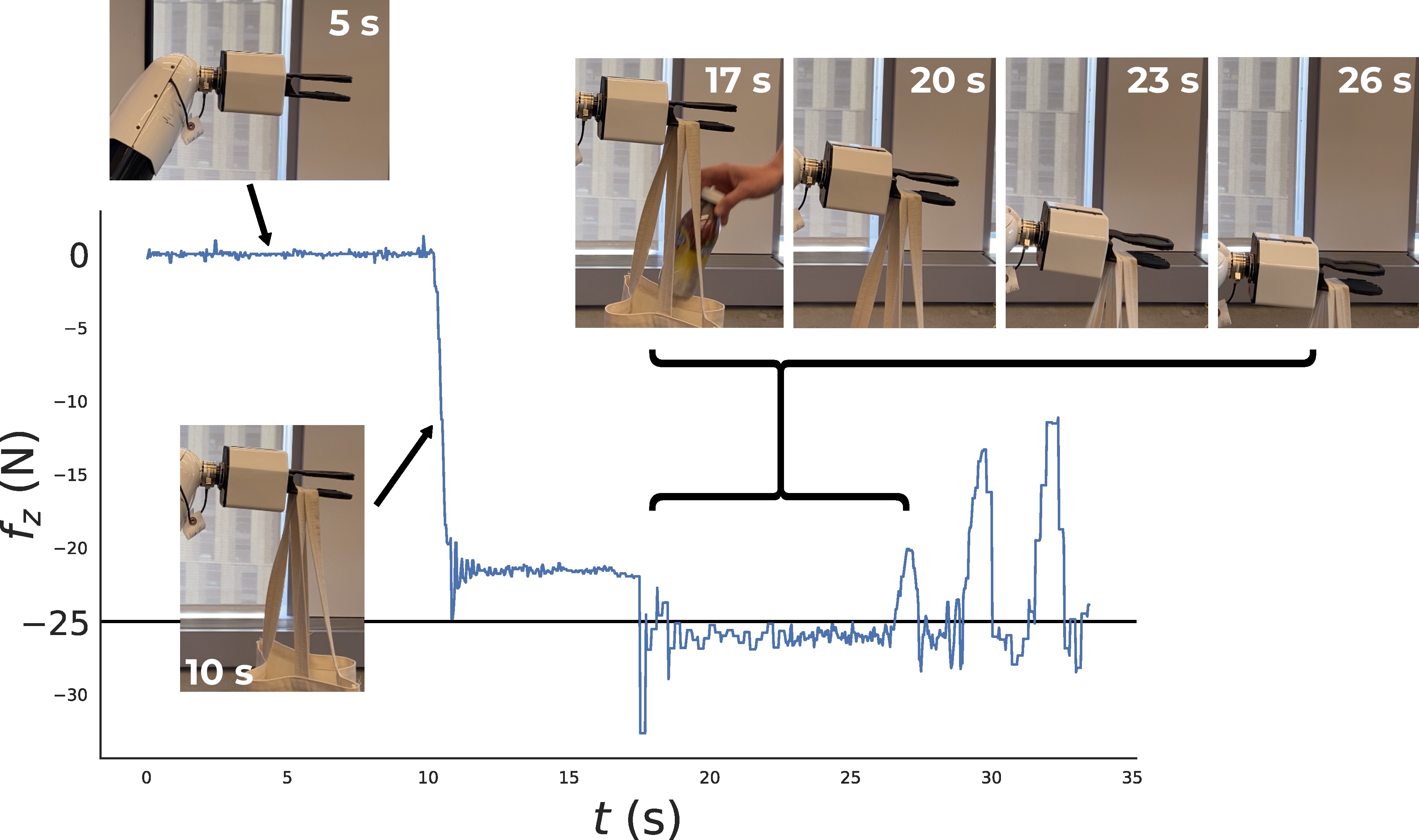}
    \caption{The measured net force along the robot's $z$ axis, showing our controller's response to applied \SI{20}{N} (applied at \SI{10}{s}) and \SI{25}{N} (at \SI{17}{s}) loads. Since we set the safety limit in the $z$ direction at \SI{25}{N}, the robot does not move until the \SI{25}{N} load is applied; once the additional load is applied it smoothly drops the load. The spikes occurring at $>$\SI{26}{s} are caused by the test mass touching the ground, de-loading the robot and causing it to try to move back towards its setpoint.}
    \label{fig:linear_test_results}
\end{figure}

\begin{figure*}[tb]
    \centering
    \includegraphics[width=\linewidth]{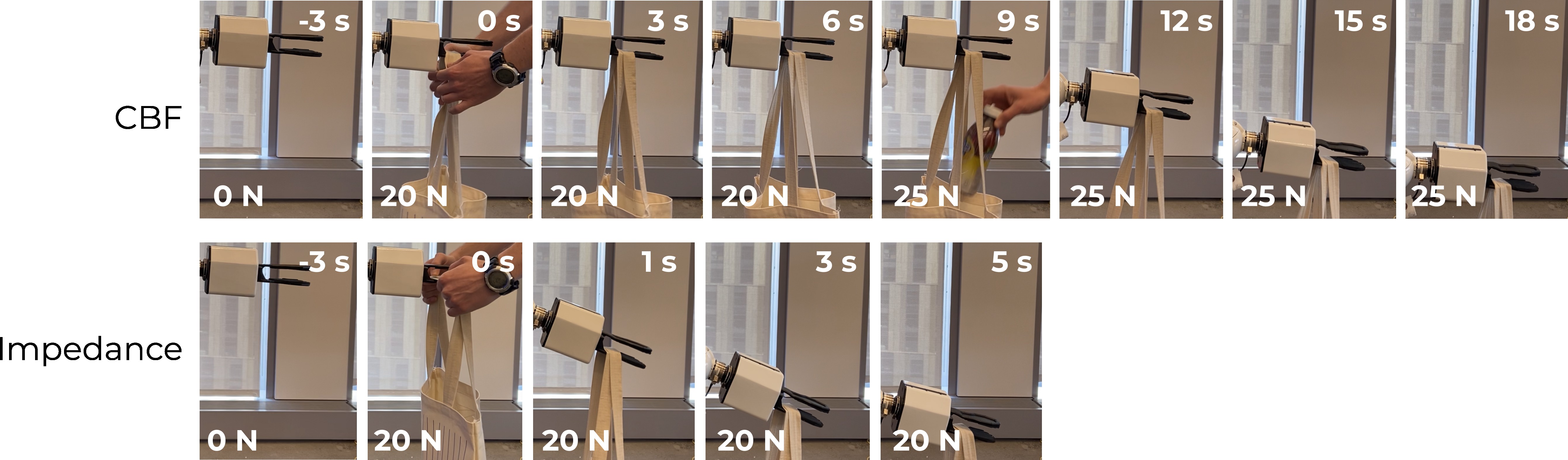}
    \caption{A comparison of our CBF controller and a traditional stiffness controller in the test described in Fig.~\ref{fig:linear_test_setup}. Our controller maintains its position when the \SI{20}{N} force is applied and only moves when the \SI{25}{N} force is applied to exceed its prescribed safety limit. In contrast, the stiffness controller fails to support the \SI{20}{N} load, despite the stiffness being tuned as high as possible without becoming unstable (\SI{40}{N/m}).}
    \label{fig:linear_test_comparison}
\end{figure*}

\subsection{Collaborative carrying task}

We can now apply our CBF force-feedback controller to the motivating human-robot collaboration example from Fig.~\ref{fig:task}. In this task, the human and the robot must collaboratively lift a bulky object; however, to avoid pushing the human off balance or causing their grasp to fail, we wish to impose safety limits on the forces and torques that the robot can apply to the carried object. We set the safety limits as follows: the force in the $x$ and $y$ directions (in the robot's frame) must not exceed \SI{10}{N}, the force in the $z$ direction (less 50\% of the carried object's weight) must not exceed \SI{10}{N}, the torques about the $y$ and $z$ axes must not exceed \SI{3}{Nm}, and the torque around the $x$ axis must not exceed \SI{0.5}{Nm}. We use CBF parameters $\lambda = 10$, $k = 1$, $\alpha = 1$ for forces, and $\alpha = 10$ for torques. To add an additional layer of responsiveness, we add a simple inverse-dynamics controller to control the movement of the robot's base to that of a free-floating rigid body. This allows the human to not only guide the robot's end-effector but also move the robot around the room. This inverse-dynamics controller only affects the motion of the base; the arm and torso motions are driven by the CBF controller, and we include a video of our controller operating without any base motion in our supplementary video.

Using this capability, we simulate a basic construction task where the human and the robot must collaboratively position a panel of foam insulation. We use an ArUco marker~\cite{aruco} to localize the panel and establish a grasp, but we cover this marker once the robot has grasped the panel (it is not used while carrying; only force and joint position measurements are used). We also program the robot to recognize when the human has rotated the panel by 90 degrees and lowered it to the ground; this is the signal to release the board and end the test.

\begin{figure*}[tb]
    \centering
    \includegraphics[width=\linewidth]{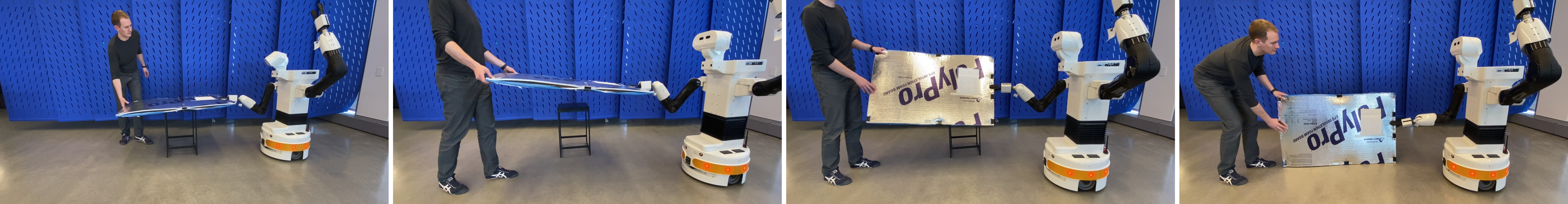}
    \caption{Using our CBF-based force-feedback controller to control the right arm and torso of the TIAGo robot (the base is controlled separately with a velocity controller that makes the base behave like a free-floating, damped rigid body). Our supplementary video shows the robot responding to push/pull, side-to-side, and twisting motions from the human; our CBF controller allows the robot to respond without applying too much force to the carried object, which could cause either the human's or the robot's grasp to fail.}
    \label{fig:carrying_snapshots}
\end{figure*}

As shown in Fig.~\ref{fig:carrying_snapshots} and in our supplementary video, our controller enables the robot to successfully perform this collaborative carrying task, allowing the human to guide the robot through back-and-forth, side-to-side, and twisting motions of the carried panel. We found that for the majority of the task, the robot's end effector was pushed away from the desired target pose, indicating that the CBF constraints were active when the human was guiding the motion of the carried object.

\section{Limitations \& Conclusion}

In our experiments, we found that the human's qualitative experience of working with our CBF-based controller was sensitive to the bandwidth of the controller and underlying hardware. Our experiments on the TIAGo robot were limited by the bandwidth of the force-torque sensor (which publishes measurements at \SI{30}{Hz}) and the whole-body controller that we used to implement velocity control of the robot's end effector. For safety reasons, the speed of the TIAGo's end effector is limited when using the provided whole-body controller, which means that human testers reported that the robot sometimes felt ``slow'' or ``sluggish'' to work with. We anticipate that the user experience could be improved by using more performant hardware for both sensing and actuation; we do not believe that the CBF-based controller will cause issues in this domain, as previous work has shown that control frequencies of \SI{>100}{Hz} are possible with CBF-based controllers (even with un-optimized Python implementations~\cite{dawson2021safe}).

Despite these limitations, we present an easy-to-implement and simple-to-tune force feedback controller that enables safe, responsive human-robot interaction. Our controller uses control barrier functions to ensure that the robot does not exert too much force on its environment while still allowing useful contact to occur. Moreover, the use of our controller does not require any system identification of either the robot or its environment, and our controller does not require any sensor feedback beyond joint positions and wrench measurements at the end effector (both of which are commonly available for commercial robotic arms). We use our controller to accomplish a simulated human-robot collaborative construction task, cooperatively maneuvering a large bulky object without subjecting either the human or the carried object to excessive force. We hope that our controller will help enable solutions to a wide range of similar problems where safe, responsive force-feedback can help robots collaborate more effectively with humans.

\section*{Acknowledgments}

We would like to sincerely thank the MIT-IBM Watson AI Lab team for providing access to and support for the TIAGo robot, particularly Jason Phillips, Luke Inglis, Mark Weber, and David Cox.

\bibliographystyle{IEEEtran}
\bibliography{main}

\end{document}